\documentclass[letterpaper, 10 pt, conference]{ieeeconf}  
\IEEEoverridecommandlockouts                              
\usepackage{multicol}
\usepackage{mathtools} 
\usepackage{amssymb} 
\usepackage{fancyhdr}
\usepackage{dsfont}
\let\labelindent\relax

\usepackage{amsthm}
\usepackage{diagbox}
\usepackage{arydshln}
\usepackage{algorithm}
\usepackage{algpseudocode}
\usepackage{subfigure} 
\usepackage[T1]{fontenc}
\usepackage{array}
\usepackage{booktabs}
\usepackage{times}
\usepackage{graphicx}
\usepackage[font=footnot
esize]{caption}
\usepackage{wrapfig}
\usepackage{tikz}
\usepackage{enumitem}
\usepackage{amsmath}
\usepackage[dvipsnames]{xcolor}
\usetikzlibrary{arrows.meta,positioning,calc}
\usetikzlibrary{patterns,angles,quotes,arrows}
\usetikzlibrary{positioning}

\usepackage{cancel} 

\definecolor{myRed}{RGB}{235,81,73}
\definecolor{myOrange}{RGB}{241,146,62}
\definecolor{myGold}{RGB}{251,235,142}
\definecolor{myYellow}{RGB}{255,255,84}
\definecolor{myGreen}{RGB}{58 124 66}
\definecolor{myGray}{RGB}{64,64,64}

\newtheorem{mytheorem}{Theorem}
\newtheorem{mydefinition}{Definition}
\newtheorem{myassumption}{Assumption}

\newtheorem{lemma}{Lemma}

\newtheorem{corollary}{Corollary}

\newcommand\numberthis{\addtocounter{equation}{1}\tag{\theequation}}

\usepackage{microtype}

\makeatletter
\let\NAT@parse\undefined
\makeatother
\definecolor{mycitecolor}{RGB}{71, 191, 38}
\definecolor{mylinkcolor}{RGB}{40, 115, 201}
\usepackage[bookmarks=true, colorlinks, citecolor=mycitecolor,linkcolor=mylinkcolor,urlcolor=mycitecolor]{hyperref}  

\title{\LARGE \bf
Learning Control Policies to Provably Satisfy Hard Affine Constraints for Black-Box Hybrid Dynamical Systems
}

\author{Aayushi Shrivastava, Kartik Nagpal, Sairam Jinkala, Jean-Baptiste Bouvier, and Negar Mehr
\thanks{*This work is supported by the National Science Foundation, under grants ECCS-2438314 CAREER Award, CNS-2529645, and by the grant from Boeing.}
\thanks{All authors are with the Department of Mechanical Engineering, University of California Berkeley, Berkeley, CA 94709, USA {\tt\small aayushis@berkeley.edu}}%
}

\begin{document}
\bstctlcite{IEEEexample:BSTcontrol}

\maketitle
\thispagestyle{empty}
\pagestyle{empty}

\begin{abstract}

Ensuring safety for black-box hybrid dynamical systems presents significant challenges due to their instantaneous state jumps and unknown explicit nonlinear dynamics. Existing solutions for strict safety constraint satisfaction, like control barrier functions (CBFs) and reachability analysis, rely on direct knowledge of the dynamics. Similarly, safe reinforcement learning (RL) approaches often rely on known system dynamics or merely discourage safety violations through reward shaping. In this work, we want to learn RL policies which provably satisfy affine state constraints in closed loop for black-box hybrid dynamical systems with affine reset maps. Our key insight is forcing the RL policy to be affine and repulsive near the constraint boundaries for the unknown nonlinear dynamics of the system, providing guarantees that the trajectories will never violate the constraint. We further account for constraint violation due to instantaneous state jumps that occur due to impacts or reset maps in the hybrid system by introducing a second repulsive affine region before the reset that prevents post-reset states from violating the constraint. We derive sufficient conditions under which these policies satisfy safety constraints in closed loop. We also compare our approach with state-of-the-art reward shaping and learned-CBF methods on hybrid dynamical systems like the constrained pendulum and paddle juggler environments. In both scenarios, we show that our methodology learns higher quality policies while always satisfying the safety constraints.

\end{abstract}


\color{black}

\section{INTRODUCTION}\label{introduction}

Hybrid dynamical systems are fundamental to modeling real-world systems that exhibit both continuous and discrete-event dynamics, with applications spanning robotics, aeronautics, and logistics \cite{Lygeros2008HybridSystems, ricardo2009hybrid}. The nonlinear and discrete-event nature of such systems makes their control particularly challenging. Reinforcement Learning (RL) has shown promise in handling such dynamics, which makes it a natural candidate for learning control policies in hybrid systems where analytical controller design is often intractable.

However, the traditional RL paradigm assumes no explicit knowledge of the underlying dynamics and relies instead on access to a simulator. Consequently, it is challenging to formally guarantee safety for such systems. In most safe RL formulations, safety is encoded through state constraints that prevent the system from entering unsafe regions. Many state-of-the-art algorithms incorporate penalty terms in the reward to discourage violations \cite{CPO, PPO-Barrier}, yet these methods lack theoretical assurances that constraints will be satisfied during deployment. Moreover, the instantaneous state changes caused by reset maps in hybrid systems can drive the system into unsafe regions within a single timestep, which makes constraint satisfaction particularly difficult to ensure under unknown dynamics.


\tikzset{
  mode/.style       ={circle,draw,thick,minimum size=22mm},
  smallmode/.style  ={circle,draw,thick,minimum size=8mm},
   mode_qi/.style    ={circle,draw,thick,minimum size=22mm, fill=blue!20, fill opacity=0.3, text opacity=1},
  mode_qj/.style    ={circle,draw,thick,minimum size=22mm, fill=purple!20, fill opacity=0.3, text opacity=1},
  lab/.style        ={font=\large},
  slab/.style       ={font=\small},
  ctrl/.style       ={-{Straight Barb[length=3pt]}, very thick}, 
}
\begin{figure}
    \centering
            \begin{tikzpicture}[scale=0.7]
                \def\ymin{2}    
                \def\ymax{5}    
                \def\ydotmax{2} 
                \def\ydotmaxxxx{2.7}
                \def\tick{0.08} 
                \def\width{6} 
                \def\height{\ydotmax + 1.2} 
                \def\arrowLen{0.6}
                
                \fill[blue!20, opacity=0.3] (-\width/2, -0.45*\ydotmax) rectangle (0, \height);
                \fill[purple!20, opacity=0.3] (0, -0.45*\ydotmax) rectangle (\width, \height);
        
                \filldraw[fill=brown!50, draw=white] (\ymin, 0) -- (\ymax, 0) -- (\ymin, \ydotmax) -- (\ymin, 0);
                \draw[myOrange, ultra thick, dotted] (\ymax, 0) -- (\ymin, \ydotmax);
                \node at (1.3*\ymin, 0.4*\ydotmax) {\textcolor{black}{$\mathcal{B}_J$}};
        
                \filldraw[fill=myYellow, draw=white] (-\width/2, \ydotmaxxxx) -- (-\width/2,  \ydotmaxxxx-0.5) -- (0, \ydotmaxxxx-0.5) -- ( 0, \ydotmaxxxx);
                \node at (-\width/4, 2.4) {\textcolor{myRed}{$\mathcal{B}$}};
        
                \node at (\width, \height) {\textcolor{gray}{$S$}};
                \node at (-\width/4, \height) {\textcolor{gray}{$q_i$}};
                \node at (\width/2, \height) {\textcolor{gray}{$q_j$}};
                \draw[very thick,gray, ->] (-\width/2, 0) -- (\width, 0);
                \node at (\width+0.30, 0) {\textcolor{gray}{$s$ }}; 
                \draw[very thick, gray, ->] (0, -0.45*\ydotmax) -- (0, \height);
                \node at (-0.25, \height+0.20) {\textcolor{gray}{$\dot s$ {\color{red}}}}; 
                \node at (-0.25, -0.2) {\textcolor{gray}{$0$}};
                \draw[ultra thick, dashed, black] (\ymax, -0.45*\ydotmax) -- (\ymax, 0);
                \draw[ultra thick, dashed, myOrange] (\ymax, -0.15) -- (\ymax, \height);
        
                \draw[ultra thick, solid, myRed] (-\width/2, \ydotmaxxxx) -- (0, \ydotmaxxxx);
        
                \draw[thin, dashed, black] (-\width/2, \ydotmaxxxx/3) -- (0, \ydotmaxxxx/3);

                \draw[very thick, black] (-\tick, \ydotmax) -- (\tick, \ydotmax);
                \node at (-0.45, \ydotmax- 0.2) {$\dot s_{Jmax}$};
                \draw[very thick, black] (\ymin, -\tick) -- (\ymin, \tick);
                \node at (\ymin, -0.21) {$s_{Jmin}$};
                \draw[very thick, black] (\ymax, -\tick) -- (\ymax, \tick);
                \node at (\ymax+0.5, \height-0.4) {\textcolor{myOrange}{$y_{J}$}};
                \node at (\ymax+0.5, -0.4*\ydotmax) {\textcolor{black}{$G_{ji}$}};
                \draw[very thick, black] (-\tick, \ydotmaxxxx) -- (\tick, \ydotmaxxxx);
                \node at (0.2, \ydotmaxxxx) {\textcolor{red}{$y$}};
                \node at (-\width/2, \ydotmaxxxx/3-0.4) {\textcolor{black}{$R_{ji}(s)$}};
                
                \foreach \x in {-0.3*\width/2, -0.9*\width/2}
                    \def\len{0.7*\arrowLen * \y}
                    \draw[thick,-{Straight Barb[length=3pt]}] (\x, \ydotmaxxxx/3+0.1) -- (\x +0.2,\ydotmaxxxx/3 +0.4);
                \foreach \x in {-0.5*\width/2,-0.7*\width/2}
                    \draw[thick, -{Straight Barb[length=3pt]}] (\x, \ydotmaxxxx/3+0.1) -- (\x - 0.3,\ydotmaxxxx/3 +0.4);
        
                \foreach \x in { -0.5*\width/2,-0.7*\width/2, -0.9*\width/2}
                    \def\len{0.7*\arrowLen * \y}
                    \draw[thick, -{Straight Barb[length=3pt]}] (\x, \ydotmaxxxx/3+0.9) -- (\x - 0.1,\ydotmaxxxx/3 +1.2);
        
                \foreach \x in {-0.3*\width/2,-0.7*\width/2,-0.9*\width/2}
                    \def\len{0.7*\arrowLen * \y}
                    \draw[thick, -{Straight Barb[length=3pt]}] (\x, \ydotmaxxxx -0.05) -- (\x,\ydotmaxxxx-0.44);
        
                \foreach \y in {0.3*\ydotmax, 0.5*\ydotmax, 0.7*\ydotmax, 0.9*\ydotmax}
                    \def\len{0.7*\arrowLen * \y}
                    \draw[thick, -{Straight Barb[length=3pt]}] (0.4*\ymin -0.5*\len, \y) -- (0.4*\ymin + 0.5*\len, \y);
        
                \foreach \y in {0.3*\ydotmax, 0.5*\ydotmax, 0.7*\ydotmax, 0.9*\ydotmax}
                {
                    \def\len{0.7*\arrowLen * \y}
                    \draw[thick, -{Straight Barb[length=3pt]}] (\ymin-0.5*\len, \y) -- (\ymin, \y) -- (\ymin + 0.5*\len, \y - 0.33*\len);
                }
                    
                \foreach \y in {0.2*\ydotmax, 0.4*\ydotmax, 0.6*\ydotmax}
                    \def\len{0.7*\arrowLen * \y}
                    \draw[thick, -{Straight Barb[length=3pt]}] (1.5*\ymin - 0.5*\len, \y) -- (1.5*\ymin + 0.5*\len, \y - 0.67*\len);
                \def\y{0.4*\ydotmax}
                \def\len{0.7*\arrowLen * \y}
                \draw[thick, -{Straight Barb[length=3pt]}] (1.8*\ymin - 0.5*\len, \y-0.18*\ydotmax) -- (1.8*\ymin + 0.5*\len, \y-0.18*\ydotmax - 0.67*\len);
                \def\y{0.6*\ydotmax}
                \def\len{0.7*\arrowLen * \y}
                \draw[thick, -{Straight Barb[length=3pt]}] (2*\ymin - 0.5*\len, \y-0.3*\ydotmax) -- (2*\ymin + 0.5*\len, \y-0.3*\ydotmax - 0.67*\len);
                
                \foreach \x in {1.3*\ymin, 1.62*\ymin, 1.94*\ymin, 2.25*\ymin}
                    \draw[thick, -{Straight Barb[length=3pt]}] (\x, 0.08*\ydotmax) .. controls (\x+0.08*\ymin, 0) .. (\x+0.16*\ymin, -0.08*\ydotmax);
        
        
                    
            \end{tikzpicture}

    \vspace{-5pt}
    \caption{
    Phase portrait of the affine safety constraint $y$. The affine repulsive buffer $\mathcal{B}$ (yellow) ensures that trajectories in mode $q_i$ approaching the constraint $y=Cs \leq d$ (\textcolor{myRed}{red}) cannot exit the buffer. For states in mode $q_j$ that reach the guard condition $G_{ji}$, the reset map $R_{ji}(s)$ may however lead to the violation of $y$ in mode $q_i$. To prevent this, the buffer $\mathcal{B}_J$ (brown) provides a repulsive region that dissipates inertia and pushes trajectories away from the affine constraint $y_J = \Tilde{C}s \leq \Tilde{d}$ (\textcolor{myOrange}{orange}). Together, the two buffers guarantee constraint satisfaction for both continuous dynamics and discrete state transitions.
    }
    \label{fig:hybrid_policed_rl_sch}
    \vspace{-15pt}
\end{figure}
In this paper, we seek to learn a control policy which is guaranteed to satisfy an affine state constraint for a black-box hybrid dynamical system. We are inspired by POLICEd-RL \cite{POLICEd_RL, CDC_POLICEd_RL}, which guarantees affine constraint satisfaction for smooth nonlinear black-box systems. In POLICEd-RL, the policy is forced to be affine and repulsive within a buffer region near the constraint boundary, ensuring that controlled trajectories never cross it. However, hybrid systems also include discrete state jumps, defined by reset maps, which can lead to constraint violations even if the pre-jump state was satisfying our constraint. To ensure constraint satisfaction, despite these instantaneous jumps, our approach introduces a secondary repulsive affine buffer which prevents the system from entering unsafe pre-jump states. We assume that the reset map is affine and known. During execution, this dual affine region design for each constraint provides constraint satisfaction guarantees for our trained policies, both for continuous dynamics and across discontinuities, without requiring any explicit system knowledge. 

To achieve this behavior, we need to learn a single RL policy that simultaneously exhibits affine and nonlinear behavior dependent on the system's state, which is challenging and often leads to unstable training. To address this, we propose a switched actor architecture in which our policy network alternates between two actor types based on the system state. Near constraint boundaries, an affine actor enforces safety by inducing repulsive behavior, whereas elsewhere a multi-layer perceptron (MLP) actor handles the complex nonlinear dynamics of the hybrid system. This design allows our method to combine provable constraint satisfaction with the flexibility and expressivity of deep reinforcement learning. We compare our approach to a learned control barrier function~\cite{PPO-Barrier} and a soft constraint policy optimization baseline~\cite{CPO} and show that our method strictly enforces safety while achieving competitive performance.

In summary, our contributions in this work are as follows:
\begin{enumerate}[itemsep=0pt, topsep=2pt, parsep=0pt, partopsep=0pt]
    \item We propose an RL framework that guarantees satisfaction of affine state constraints in black-box hybrid dynamical systems with affine reset maps. 
    \item We establish sufficient conditions that ensure a learned policy satisfies affine safety constraints under hybrid dynamics with affine reset maps.
    \item We introduce a switched actor network architecture that employs distinct actor types in different regions of the state space: an affine actor near constraint boundaries to ensure safety, and an MLP actor elsewhere for expressivity.
    \item We demonstrate our framework on numerical simulations of a constrained pendulum and a one-dimensional paddle juggler and show that our method strictly enforces safety while achieving competitive performance relative to baseline RL methods.
\end{enumerate}

The remainder of this paper is organized as follows: In Section~\ref{related_works}, we review related works. In Section~\ref{framework}, we introduce our problem formulation. In Section~\ref{hybrid_policed_rl}, we provide details on our proposed approach and present theoretical guarantees. In Section~\ref{numerical_simulations}, we report our numerical results and compare with two baselines. Finally, we conclude the paper and discuss future directions in Section~\ref{conclusion}.

\section{Related Work} \label{related_works}
\subsection{Safety for Hybrid Systems}
A dominant paradigm in safety-critical control is the use of Control Barrier Functions (CBFs) \cite{Ames2019CBFReview}, which are Lyapunov-like functions that guarantee forward invariance of safe sets. In our case, these safe sets are the parts of the state space that satisfy the given safety constraints. Extensions to hybrid systems include computing local CBFs for each continuous mode \cite{yang2024safe}, learning CBFs from data \cite{yang2024learninglocalcontrolbarrier}, and combining CBFs with Model Predictive Control (MPC) \cite{Agrawal2017DiscreteCB,hybrid_mpc}. Computing a global CBF has also been proposed in \cite{Prajna2004BarrierCertificates}, which provides safety guarantees across discrete transitions. However, these approaches assume precise knowledge or an accurate estimate of system dynamics, an assumption that may not hold in general in RL or many real-world systems.

Another line of work is reachability analysis, which computes the set of states reachable from initial conditions that satisfy safety constraints \cite{Alur1994TimedAutomata, Alur1991HybridAutomata}. Hamilton-Jacobi (HJ) methods have been used for nonlinear uncertain systems \cite{Bansal2017HJIOverview} and hybrid systems \cite{Borquez2023HamiltonJacobiRA}, and several works combine HJ reachability with RL \cite{Fisac2019HJRL, bansal2021deepreach, chilakamarri2025reachability} via value or Hamiltonian function approximations. However, like CBF-based approaches, these methods require the precise knowledge of the system dynamics. Our work addresses this by developing a model-free approach to guarantee constraint satisfaction in black-box hybrid systems, assuming access to a simulator and known reset maps.

\subsection{Constraints in Reinforcement Learning}

RL has become a prevalent approach for learning optimal control policies, particularly because it does not require explicit knowledge of system dynamics. Yet a primary limitation of model-free RL is the lack of safety guarantees, essential for deployment in safety-critical settings. The field of safe RL has produced numerous approaches with various models for safety \cite{DulacArnold2021Challenges, Brunke2022SafeLearning}. A common formulation is the Constrained MDP (CMDP) \cite{Altman2021ConstrainedMDP}, with variants including state-wise constrained MDPs \cite{Zhao2023StatewiseCPO}, Constrained Policy Optimization (CPO) \cite{CPO}, and state-wise constrained policy optimization \cite{Zhao2023LearnWithImagination}. These methods typically penalize constraint violations through reward adjustments but do not provide formal guarantees of constraint satisfaction during deployment \cite{Gu2022SafeRLReview}.



Another approach to incorporating safety into RL leverages learned CBFs \cite{Ames2019CBFReview}. For example, ConBaT \cite{Meng2023ConBaTCB} trains a control barrier transformer to avoid unsafe actions, and several works propose learning safety certificates directly from data \cite{Qin2022Sablas, Ma2022JointSynthesis, PPO-Barrier, guassian_CBF}. Recent works \cite{pmlr-v283-mestres25a,BouvierJ-RSS-25} utilize safe gradient flows to guarantee that policy updates satisfy safety constraints at every training iteration and deployment. However, many of these data-driven approaches aim for probabilistic safety certificates in stochastic environments, whereas we focus on hard-constraint guarantees in deterministic settings.



MPC has also been used to impose safety constraints in RL \cite{Hewing2020LBMPC}, by predicting safe actions in a receding-horizon manner by leveraging dynamics models to anticipate future behavior. For black-box systems, this requires learning either robust \cite{Aswani2013SafeMPC, DiPalma2004MultiModelMPC} or stochastic \cite{Lorenzen2017StochasticMPC} dynamics models. MPC has also been used as a safety filter on learned policies \cite{Wabersich2021SafetyFilter, Hewing2020LBMPC}. However, solving a high-dimensional optimization at every timestep limits scalability, whereas closed-loop RL policies operate at much lower computational cost during execution.

\section{Problem Formulation} \label{framework}

In this section, we formally define what we mean by safety for a black-box hybrid dynamical system. We model our hybrid system via the framework of hybrid automata \cite{Lygeros2008HybridSystems}, and enforce safety by formulating a constraint satisfaction problem that prevents trajectories from entering unsafe regions. We consider a hybrid dynamical system of the form
\vspace{-8pt}
\begin{equation}\label{eq:hybrid_automata}
    H = \big((Q \times S), U, f, \rho_0, T, G, R\big),\vspace{-2pt}
\end{equation}
where $Q:= \{q_1,q_2,...,q_N\}$ is the finite set of discrete modes of the system, $s \in S \subset \mathbb{R}^n$ is the continuous state of the system, and $u \in U \subset \mathbb{R}^m$ is the continuous control input. Each discrete mode $q_i \in Q$ has deterministic continuous dynamics $f_{i}: Q \times S \times U \rightarrow \mathbb{R}^n$. In this paper, our dynamics are implicitly black-box, meaning we do not have access to the analytical form $\dot{s} = f_{i}\left (s(t),u(t) \right )$, but we can evaluate $f_i$. This aligns with common RL setups, where we can sample a simulator or a physical system that encodes $f_{i}$.
\tikzset{
  mode/.style       ={circle,draw,thick,minimum size=22mm},
  smallmode/.style  ={circle,draw,thick,minimum size=8mm},
   mode_qi/.style    ={circle,draw,thick,minimum size=22mm, fill=blue!20, fill opacity=0.3, text opacity=1},
  mode_qj/.style    ={circle,draw,thick,minimum size=22mm, fill=purple!20, fill opacity=0.3, text opacity=1},
  lab/.style        ={font=\large},
  slab/.style       ={font=\small},
  ctrl/.style       ={-{Straight Barb[length=3pt]}, very thick}, 
}
\begin{figure}
    \centering
            \begin{tikzpicture}[scale=0.5]
                \scriptsize
                \coordinate (C) at (0,5);
                \node[mode_qi]  (Qi) at (C) {};
                \node at ($(C)+(0, 1.2)$) {$\mathbf{q_i}$};
                \node at ($(C)+(0, 0.30)$) {$\dot{s}=f_i(s,u)$};
                \node at ($(C)+(0,-0.20)$) {$s\in S_i$};
                \node at ($(C)+(0,-0.7)$) {$y=Cs < d$};
                
                \coordinate (R1) at ( 5.0,  5.0);
                
                \node[mode_qj] (Qj) at (R1) {};
                \node at ($(R1)+(0, 1.2)$) {$\mathbf{q_j}$};
                \node at ($(R1)+(0, 0.70)$) {$\dot{s}=f_j(s,u)$};
                \node at ($(R1)+(0, 0.10)$) {$s\in S_j$};
                \node at ($(R1)+(0, -0.5)$) {$y=Cs < d$};
                \node at ($(R1)+(0,-1.0)$) {$y_J=\Tilde{C}s < \Tilde{d}$};
                \node at ($(R1)+(0,-1.5)$) {$\forall s \in J_{ji}$};
    
                
                \draw[ctrl] (Qi.north) to[bend left=20] 
                node[pos=0.45,below,slab] {$s'=R_{ij}(s)$} 
                   node[pos=0.45,above,slab] {$G_{ij}$} (Qj.north); 
                
                \draw[ctrl] (Qj.south) to[bend left=30]
                  node[midway,above,slab] {$s'=R_{ji}(s)$}  
                  node[pos=0.45,below,slab] {$G_{ji}$}(Qi.south);
        
            \end{tikzpicture}
        
    \vspace{-10pt}
    \caption{
    Hybrid automaton representation of a two-mode hybrid system with continuous dynamics $f_i$ and $f_j$ within discrete modes $q_i$ and $q_j$ respectively,  transition between mode $q_i$ to $q_j$ is given by guard condition $G_{ij}$ and reset map $R_{ij}$. The constraint $y=Cs \leq d$ is active in both discrete modes whereas the resulting jump constraint $y_J = \Tilde{C}s \leq \Tilde{d}$ is active in mode $q_j$. 
    }
    \label{fig:hybrid_automata}
    \vspace{-12pt}
\end{figure}
%
%

We denote $\rho_0$ as the distribution of initial states. We denote the state space of mode $q_i \in Q$ as $S_i\subset S$. Whenever our deterministic dynamics result in a next state $s'$ which exits the valid operation domain for mode $q_i$, the system takes a forced transition into another discrete mode $q_j$. We capture discrete transitions between modes by the relation $T \subseteq Q \times S \times Q \times S$, where $((q_i,s),(q_j,s')) \in T$ indicates a valid jump from state $(q_i,s)$ to $(q_j,s')$. The guard set $G(q_i,q_j) = \{s \in S : ((q_i,s),(q_j,s')) \in T\}$ denotes the states in which transitions occur from mode $q_i$ to $q_j$, as shown in Fig.~\ref{fig:hybrid_automata}. The reset (or impact) map $R(s)$ defines the resulting state after a transition, restricted to the guard set $G(q_i,q_j)$, meaning
\vspace{-2pt}
\begin{equation}\label{eq:reset_map}
    R(s):= \{s' \in S : ((q_i,s),(q_j,s')) \in T \}.\vspace{-2pt}
\end{equation}
We simplify our notation for $G(q_i,q_j)$ as $G_{ij}$ and choose to denote $R(s)$ as $R_{ij}(s)$ for transitions from mode $q_i$ to $q_j$. We assume that the reset map is affine, a common characteristic in many hybrid systems of interest, such as air-traffic control~\cite{HybridTomlin}, aircraft autopilot modes~\cite{Lygeros2008HybridSystems}, and bipedal locomotion~\cite{grant_alip}. 


\begin{myassumption} For all transitions between discrete modes of the hybrid system, the reset map is affine. Specifically, for a transition between $q_j$ to $q_i$, we assume that the reset map is given by
    \vspace{-1pt}
    \begin{equation} \label{eq:affine_reset}
        s' = R_{ji}(s) = M_{ji} s + p_{ji}, 
    \end{equation}
where $M_{ji} \in \mathbb{R}^{n \times n}$ and $p_{ji} \in \mathbb{R}^n$ are known.
\vspace{-1pt}
\end{myassumption}
\vspace{-2pt}
Note that we assume prior knowledge of the guard condition $G_{ji}$ and reset map $R_{ji}(s)$. Be aware that this assumption is not overly restrictive as in practice we can estimate $R_{ji}(s)$ or $R_{ij}(s)$ and $G_{ji}$ or $G_{ij}$ from data samples, especially by running the system with any control policy and finding the states that show a discrete jump in the trajectory. 

We define safety as satisfying a constraint on the system's output, irrespective of the mode of operation. As such, we can define the safe set as the subset of states where the output constraint holds, and the unsafe set will be its complement. This constraint definition implicitly defines a boundary between safe and unsafe states, analogous to control barrier functions. Therefore, if a system trajectory never violates the constraint, then all states in the trajectory are within the safe set. 

In our framework, we focus on affine constraints. Such constraints often arise in robotics, such as center-of-mass stabilization in locomotion tasks via velocity control, joint limits, and workspace boundaries \cite{grant_alip}, and in autonomous driving, through linearized inter-vehicle separation \cite{Lygeros2008HybridSystems}. 

To begin, we consider a single affine constraint of relative degree 1 on the system output. The relative degree of a constraint is the smallest number of times the constraint function must be differentiated with respect to time before the control input can appear in the resulting expression~\cite{Brreden2021HighRDCBF}. Therefore, the higher the relative degree, the more inertia the constraint has and the more challenging it is to satisfy \cite{Brreden2021HighRDCBF}.
\vspace{-1pt}
\begin{myassumption} 
    We assume that the system safety constraint is captured by a single affine inequality constraint on the system output $y(t) \in \mathbb{R}$ of the form
\vspace{-2pt}    
    \begin{equation}\label{eq: constraint}
        y(t) := C s(t) \leq d  \quad  \forall \ t \geq 0, 
    \end{equation}
    where $C \in \mathbb{R}^{1 \times n}$ and $d \in \mathbb{R}$ are known.
    \vspace{-3pt}
\end{myassumption}
\vspace{-4pt}
To control the hybrid system, we consider a deterministic feedback policy $u(t) = \pi_\theta\big(s(t) \big) \in U$ and model it using a deep neural network parameterized by $\theta$. Our objective is to train a policy $\pi_\theta$ such that the closed-loop system satisfies constraint~\eqref{eq: constraint} while maximizing the expected reward 
\vspace{-2pt}
\begin{equation}\label{eq: expected reward}
    \underset{\theta}{\max}\, g(\pi_\theta) := \hspace{-2mm} \underset{s_0 \sim \rho_0}{\mathbb{E}} \hspace{-1mm} \int_0^{\infty} \hspace{-3mm} \gamma^t r\big( s(t), \pi_\theta(s(t))\big) dt \hspace{5mm} \text{s.t.}\ \eqref{eq: constraint}. \hspace{2mm}\vspace{-2pt}
\end{equation}
where $\gamma \in (0,1]$ is the discount factor, $r$ is the reward function, and $\rho_0$ the distribution of initial states. The only stochasticity in our setting comes from the initial state sampling $s_0\sim\rho_0$. We consider an infinite-horizon objective, as is standard in the RL formulation we build upon. The goal is to learn a stationary policy that performs well over an indefinitely long horizon. In summary, we want to learn a deterministic RL policy that, once trained, satisfies the affine constraint~\eqref{eq: constraint} at all times, while maximizing the expected reward for a black-box hybrid dynamical system~\eqref{eq:hybrid_automata} with a known affine reset map~\eqref{eq:affine_reset}.

\section{Our Framework} \label{hybrid_policed_rl}

In this section, we present our method for solving our formulated problem. First, we discuss how we can ensure safety for each mode of the hybrid system, and then discuss how we can ensure the overall safety of the hybrid system by accounting for potential state transitions and jumps. Finally, we describe our reinforcement learning pipeline to learn such safe policies.

\subsection{Safety in a Discrete Mode}
Intuitively, when a system approaches a constraint, if the policy steers the system trajectory away from this constraint boundary, we maintain safety. As a result, if we define a buffer region preceding the safety constraint and design the policy to behave repulsively within this buffer, we will prevent any trajectories from crossing the buffer, and therefore, from violating the safety constraint. If we choose an affine policy for the buffer region, we can easily characterize the repulsive nature and provide safety guarantees.

To start, we will ensure that trajectories that remain within one mode $q_i$ of the system do not violate the safety constraint. To ensure constraint satisfaction within a single continuous mode of a hybrid system without any discrete transitions, we build on the previous work \cite{POLICEd_RL} where the authors forced the output of a deep neural network to be affine in a buffer near the constraint to provably satisfy the constraint. Inspired by their work, we define a repulsive buffer for constraint~\eqref{eq: constraint}. Given a `buffer width' $w>0$, we can define a buffer
\vspace{-2pt}
\begin{equation}\label{eq: buffer}
     \mathcal{B}(w) := \{s \in S_i : Cs \in [d-w,d]\} ,
     \vspace{-2pt}
\end{equation}
which prescribes a convex polytope, as proved in \cite{POLICEd_RL}. If well designed, all trajectories close to the constraint must enter this buffer $\mathcal{B}(w)$, and if the buffer is sufficiently repulsive, the trajectories will never reach and violate the constraint boundary. We provide an illustrative example in Fig.~\ref{fig:hybrid_policed_rl_sch}. We choose our buffer width $w$ such that no trajectory can traverse the entire buffer within a single timestep. To increase readability, we denote $\mathcal{B}(w)$ as $\mathcal{B}$. And while the dynamics of a given mode $f_i$ may be nonlinear, we can utilize an affine approximation of the dynamics $f_i$ to derive a sufficiently repulsive condition for our buffer $\mathcal{B}$ that guarantees safety. To begin, we define the following parameter $\varepsilon_i$ which measures how far the true system dynamics is from being linear:
\vspace{-5pt}
\begin{mydefinition}
    An approximation measure $\varepsilon_i$ of dynamics $f_i$ with respect to constraint \eqref{eq: constraint} and buffer \eqref{eq: buffer} is any $\varepsilon_i \geq 0$ for which there exist any matrices $A \in \mathbb{R}^{n \times n}$, $B \in \mathbb{R}^{n \times m}$ and $c \in \mathbb{R}^n$ such that
    \vspace{-2pt}
    \begin{equation} \label{eq: epsilon}
      |Cf_{i}(s,u) - C(As + Bu + c)| \leq \varepsilon_i,\quad  \forall s \in \mathcal{B}, \forall u \in U. 
      \vspace{-2pt}
    \end{equation}
\end{mydefinition}

Note that we will require only the knowledge of $\varepsilon_i$ and not $A$, $B$, and $c$ to guarantee the satisfaction of the constraint~\eqref{eq: constraint}. Importantly, we only ensure that the affine approximation remains locally close to the true dynamics $f_i$. We can compute the estimate for $\varepsilon, A, B$ and $c$ using system identification techniques like linear least squares. We can use this affine approximation measure $\varepsilon_i$ to derive sufficient conditions that make the policy repulsive enough inside the buffer $\mathcal{B}$ to guarantee constraint satisfaction. We denote the set of vertices of $\mathcal{B}$ as $\mathcal{V}(\mathcal{B})$. Note that since our buffer $\mathcal{B}$ is a convex polytope, if it is sufficiently repulsive at the vertices $v \in \mathcal{V}(\mathcal{B})$, it is also sufficiently repulsive for all the states within the buffer. The following theorem from \cite{POLICEd_RL} establishes sufficient conditions for buffer $\mathcal{B}$ to be repulsive for nonlinear dynamics $f_i$ within mode $q_i$.
\vspace{-2pt}
\begin{mytheorem} \label{thm:theorem_policed}
    If for some approximation measure $\varepsilon_i$, repulsion condition
\vspace{-10pt}
    \begin{equation} \label{eq: repulsive}
        Cf_{i}(v,\pi_{\theta}(v)) \leq -2\varepsilon_i,
    \end{equation}
    holds for all $v \in \mathcal{V}(\mathcal{B})$, then, the trajectories will never violate the constraint 
    \eqref{eq: constraint}.
    \vspace{-5pt}
\end{mytheorem}
\vspace{-2pt}
We reproduce the proof from~\cite{POLICEd_RL} to explain how the Theorem~\ref{thm:theorem_policed} applies to hybrid system and it helps explain the proof of
Theorem~\ref{thm:theorem_reldeg2} which extends the vertex-wise argument to discrete
jumps. 
\begin{proof}
The intuition behind this proof is to use~\eqref{eq: repulsive} and approximation~\eqref{eq: epsilon} to show that $C\dot{s} \leq 0$ for all $s \in \mathcal{B}$,
which in turn prevents the trajectory from crossing buffer $\mathcal{B}$ and hence from violating the constraint~\eqref{eq: constraint}. 

Let \(\dot v:=f_i(v,\pi_\theta(v))\) denote the closed-loop vector field evaluated at \(v\in\mathcal V(\mathcal B)\). Note that since $\varepsilon_i$ is an approximation measure, there exist matrices $A$, $B$ and $c$ satisfying~\eqref{eq: epsilon}, i.e. 
 \begin{equation} 
 \begin{aligned} \label{eq: cs_epsilon}
 C(Av + B\pi_{\theta}(v) + c)   & \leq |C(Av + B\pi_{\theta}(v) + c) \\ 
 & - Cf_i(v,\pi_{\theta}(v))|   + Cf_i(v,\pi_{\theta}(v)) \\
 & \leq \varepsilon_i + C\dot{v} \leq \varepsilon_i - 2\varepsilon_i \leq -\varepsilon_i
 \end{aligned}
 \end{equation}
where we first use the triangular inequality, then the affine approximation \eqref{eq: epsilon}. Since our buffer $\mathcal{B}$ is a convex hull with $N$ vertices $v \in \mathcal{V}(\mathcal{B})$, we can write any state within $\mathcal{B}$ as $s = \sum\limits_{k=1}^{N}\alpha_kv^k$ where  $\alpha_k \geq 0$ and $\sum\limits_{k=1}^{N}\alpha_k = 1$. Note that the policy is affine in $\mathcal{B}$, i.e., $\pi_\theta(s) = D_\theta s + e_\theta$, where $D_\theta \in \mathbb{R}^{m \times n}$ and $e_\theta \in \mathbb{R}^{m}$ are the weights and biases of the policy, respectively. Hence, we have:
 \begin{equation} 
 \vspace{-5pt}
 \begin{aligned} \label{eq: convex_hull_B}
 C(As + B\pi_{\theta}(s) + c)   &  = C(As + B(D_{\theta}s + e_{\theta}) + c)\\
&= C(A+BD_{\theta})\sum\alpha_kv^k \\
&   +C(Be_{\theta} + c)\sum\alpha_k \\
& = \sum\alpha_kC[(Av^k + B\pi_{\theta}(v^k)+ c] \\
& \leq \sum\alpha_k(-\varepsilon_i) \leq -\varepsilon_i
 \end{aligned}
 \end{equation}
 where the inequality comes from~\eqref{eq: cs_epsilon} applied on each vertex $v^k$, and $\sum$ denotes $\sum\limits_{k=1}^{N}$. Then, for any state $s \in \mathcal{B}$, we have
 \begin{equation} \label{eq: Cs_dot_neg}
\begin{aligned}
C\dot{s} &= Cf_i(s,\pi_{\theta}(s)) \\
& \leq |Cf_i(s,\pi_{\theta}(s)) - C(As + B\pi_{\theta}(s) + c)|\\
&\quad + C(As + B\pi_{\theta}(s) + c) \\
& \leq \varepsilon_i - \varepsilon_i \leq 0
\end{aligned}
\end{equation}
where we first use the triangular inequality, then~\eqref{eq: epsilon} and~\eqref{eq: convex_hull_B}. We define the safe states as $S_{si}:= \{s \in S_i : Cs < d\}$. For now, we only consider trajectories remaining in state space $S_i$, which we define as $\mathcal{\tau}^S_i(s_0, u(\cdot)) := \{s(t), s(t) \in S_i$ and follows $f_i\}$ for all $s_0 \in S_i$ and actions $u \in U$. Having proved \eqref{eq: Cs_dot_neg}, we will show that it prevents all trajectories $\mathcal{\tau}^S_i(s_0,\pi_{\theta})$ from exiting the safe set $S_{si}$ when $s_0 \in S_{si}$.

Now, suppose there exists $s_0 \in S_{si}$ whose trajectory $\tau^S_i(s_0, \pi_\theta)$ leaves $S_{si}$. Since the trajectory remains in $S_i$, there exists $T > 0$ with $s(T) \in S_i \setminus S_{si}$, implying $y(T) = Cs(T) > d$. Meanwhile, $s_0 \in S_{si}$ gives $y(0) = Cs_0 < d$. By continuity of $y$, the intermediate value theorem yields $t_2 \in (0, T]$ with $y(t_2) = d$. Let $t_0 \geq 0$ mark the final entry of the trajectory into $\mathcal{B}$, so that $s(t) \in \mathcal{B}$ for all $t \in [t_0, t_2]$.
Since $y(t) = Cs(t)$ is continuously differentiable within a mode, the mean value theorem gives $t_1 \in (t_0, t_2)$ with $\dot{y}(t_1) = (y(t_2) - y(t_0))/(t_2 - t_0)$, as both numerator and denominator are strictly positive. However, $s(t_1) \in \mathcal{B}$ and~\eqref{eq: Cs_dot_neg} require $\dot{y}(t_1) \leq 0$, a contradiction. Hence, every trajectory originating in $S_{si}$ remains in $S_{si}$.
\end{proof}
\vspace{-8pt}

Theorem \ref{thm:theorem_policed} guarantees that trajectories in $q_i$ steered by a policy $u(t) = \pi_{\theta}(s(t))$ that incorporates buffer~\eqref{eq: buffer} always satisfies constraint \eqref{eq: constraint} as long as repulsion condition \eqref{eq: repulsive} is satisfied and the system trajectories remain within the same discrete mode $q_i$ as shown in Fig. \ref{fig:hybrid_policed_rl_sch}. 


\subsection{Safety When Transitioning Between Modes}

Although Theorem~\ref{thm:theorem_policed} can guarantee safety for every individual discrete mode of a hybrid system $q_i$, it does not correctly account for the discrete event transitions. When the system undergoes a transition from mode $q_j$ to $q_i$ at guard condition $G_{ji}$, the reset map $R_{ji}(s)$ could cause the state to jump to a post-reset state which violates the constraint $y$, but has never entered the buffer $\mathcal{B}$. This can even be true for states that were safe in $q_j$ before the reset. To address this challenge, we define a constraint-violating guard set $J_{ji}$ that captures the guard states that lead to constraint violation post-reset.
 \vspace{-5pt}
\begin{mydefinition}
    For a given $G_{ji}$ and $R_{ji}(s)$, we define a constraint-violating guard set $J \subseteq G$ with $J_{ji}$ = $\{s \in S : ((q_j,s),(q_i,s')) \in T; \ CR_{ji}(s) > d \}$, i.e. the guard states which post-reset violate the constraint $Cs(t) > d$.
\end{mydefinition}
 \vspace{-5pt}
Note that to guarantee constraint satisfaction, we need to ensure that our policy will avoid the set $J_{ji}$. Since our reset map is affine, we can define a new affine constraint $y_J(t)$ which is active near the guard set $J_{ji}$ and is defined as
 \vspace{-5pt}
\begin{equation}\label{eq: constraint_J}
    y_J(t) := \Tilde{C} s(t) \leq \Tilde{d} \quad \forall \ t \geq 0, \quad \forall s(t) \in J_{ji}, 
    \vspace{-5pt}
\end{equation}
where $\Tilde{C} = CM_{ji}$, and $\Tilde{d} = d - Cp_{ji}$, and $M_{ji}$ and $p_{ji}$ denote the affine approximation of $R_{ji}(s)$ from~\eqref{eq:affine_reset}.


For simplicity, first assume that the affine constraint $y_J(t)$ has relative degree one with respect to the continuous dynamics in mode $q_j$. In this case, the control input directly affects $\dot{y}_J$, so the policy can prevent the state from approaching the constraint boundary by ensuring $\dot{y}_J \leq 0$ in a buffer $\mathcal{B}_J$ preceding the unsafe guard set $J_{ji}$. Intuitively, keeping the pre-reset state away from $J_{ji}$ prevents the reset from violating the original safety constraint. We therefore construct $\mathcal{B}_J$ analogously to $\mathcal{B}$ and choose a policy that is affine within it. Let $\varepsilon_j$ denote the affine approximation measure of $f_j$ with respect to $\Tilde{C}$ in $\mathcal{B}_J$. Applying Theorem~\ref{thm:theorem_policed} to this transformed constraint gives the following lemma.
\vspace{-5pt}
\begin{lemma}
Assume that for some approximation measure $\varepsilon_j$ of the dynamics $f_j$, and relative degree $1$, $\Tilde{C}f_{j}(v,\pi_{\theta}(v)) \leq -2\varepsilon_j$ holds for all $v \in \mathcal{V}(\mathcal{B}_J)$, then the trajectories will never violate the constraint \eqref{eq: constraint_J} and never reach jump set $J_{ji}$.
    \vspace{-5pt}
\end{lemma}
However, note that depending on the reset map $R_{ji}(s)$ and the guard condition $G_{ji}$, the relative degree of constraint $y_J(t)$ can be higher than 1. For such a constraint $y_J(t)$, we need to dissipate the inertia of the states progressing towards the constraint~\eqref{eq: constraint_J} before the states can be repulsed from the constraint, as shown in Fig.~\ref{fig:hybrid_policed_rl_sch}. We can achieve this desired by designing an appropriate buffer . 
For simplicity, let the safety constraint $y_J$ be of maximum relative degree two. To maintain constraint \eqref{eq: constraint_J} in the buffer $\mathcal{B}_J$, we need that  $\dot{y}_J \leq 0$ when the trajectory reaches the constraint boundary $y_J = \Tilde{d}$ (i.e., our controller must generate a velocity away from the constraint boundary). Unlike the relative degree one case, the buffer $\mathcal{B}_J$ must first dissipate the inertia of incoming trajectories before they reach $y_J = \tilde{d}$.


Let $\dot{y}_{J\max} > 0$ be the maximal rate of constraint $y_J$ that can be dissipated within $\mathcal{B}_J$. We design $\mathcal{B}_J$ as a convex polytope such that it dissipates the inertia of trajectories arriving at some $y_{J\min}$ with velocities $\dot {y}_J \leq \dot{y}_{J\max}$ where $y_{J\min}$ is the minimum value of constraint $y_J$ in $\mathcal{B}_J$. Consequently, we choose the vertices $\mathcal{V}(\mathcal{B}_J)$ of the buffer $\mathcal{B}_J$ such that they satisfy $y_J \in [y_{J\min}, \Tilde{d}]$ and $\dot y_J \in [\dot{y}_{J\min}, \dot{y}_{J\max}]$ shown in Fig.~\ref{fig:hybrid_policed_rl_sch} in mode $q_j$. 

With only implicit access to our dynamics, we can over-approximate the vertices of our buffer $\dot{y}_{J\min}, \dot{y}_{J\max}, y_{J\min}$ and ensure that our affine buffer policies are sufficiently repulsive. For more details on buffer design, we refer the reader to \cite{CDC_POLICEd_RL}. We now state the dissipation condition for the policy to be repulsive in a buffer $\mathcal{B}_J$ as follows:
\begin{mytheorem} \label{thm:theorem_reldeg2}
    Assume that for some approximation measure $\varepsilon_j$ of $\ddot{f}_j$ which denotes the $2^{nd}$ derivative of dynamics, and relative degree $2$, the following dissipation condition
    \begin{equation} \label{eq: repulsive_reldeg2}
        \Tilde{C}\ddot{f}_{j}(v,\pi_{\theta}(v)) \leq -2\varepsilon_j - \beta \dot{v},
    \end{equation} 
    holds for all $v \in \mathcal{V}(\mathcal{B}_J)$, where $\dot{v}$ is the $1^{st}$ time derivative of the vertex $v$, and $\beta$ is defined as     
    \begin{equation}
        \beta = \frac{\dot{y}_{J\max}}{\Tilde{d} - y_{J\min}}.
        \vspace{-5pt}
    \end{equation}
    Then, if the trajectory $s$ steered by $u(t) = \pi_{\theta}(s(t))$ enters the buffer $\mathcal{B}_J$, then it will exit the buffer $\mathcal{B}_J$ without violating the constraint \eqref{eq: constraint_J}.
\end{mytheorem}

\vspace{-5pt}
More intuitively, Theorem~\ref{thm:theorem_reldeg2} guarantees that if the dissipation condition~\eqref{eq: repulsive_reldeg2} is satisfied, trajectories entering buffer $\mathcal{B}_J$ can never exit through our constraint $y_J$. Specifically, $\beta$ signifies the maximum inertia which the buffer $\mathcal{B}_J$ can dissipate. 
This is illustrated in Fig.~\ref{fig:hybrid_policed_rl_sch}, where adhering to Theorem~\ref{thm:theorem_reldeg2} produces the bent arrows in the flow field in mode $q_j$ which prevent trajectories from violating constraint~\eqref{eq: constraint_J}.
\vspace{-2pt}
\begin{proof}
We provide a brief sketch of the proof for Theorem~\ref{thm:theorem_reldeg2} as follows. Since our buffer $\mathcal{B}_J$ is a convex polytope and our policy $\pi_\theta$ is affine, we can write the control action produced by our policy as a convex combination of the behavior at the vertices $v \in \mathcal{V}(\mathcal{B}_J)$. As such, it is sufficient to verify our dissipation condition~\eqref{eq: repulsive_reldeg2} only at the vertices, yielding the condition $\Tilde{C}(As + B\pi_{\theta}(s) + c) \leq |\Tilde{C}(As + B\pi_{\theta}(s) + c) -\Tilde{C}\ddot{f}_j(s,\pi_{\theta}(s))|  + \Tilde{C}\ddot{f}_j(s,\pi_{\theta}(s)) \leq -\varepsilon_j - \beta \dot{s}$ for all $s \in \mathcal{B}_J$. So, for any interior state $s \in \mathcal{B}_J$, we have $\ddot{y}_J = \Tilde{C}\ddot{f}_j(s,\pi_{\theta}(s)) \leq |\Tilde{C}\ddot{f}_j(s,\pi_{\theta}(s)) - \Tilde{C}(As + B\pi_{\theta}(s) + c)| + \Tilde{C}(As + B\pi_{\theta}(s) + c) \leq - \beta \dot{y}_J$. From the two bounds of our inequality, we get the key differential inequality $\ddot{y}_J \leq -\beta \dot{y}_J$. This means that within our buffer $\mathcal{B}_J$, the acceleration produced by our policy will always suitably oppose any velocity toward our constraint $y_J$, ensuring $\dot{y}_J \leq 0$ at the boundary. The trajectory, therefore, cannot exit $\mathcal{B}_J$ through the constraint side, and provably prevents it from ever entering the unsafe guard set $J_{ji}$. 
\end{proof}
\vspace{-2pt}
As shown in \cite{CDC_POLICEd_RL}, Theorem~\ref{thm:theorem_reldeg2} can be naturally extended to relative degrees higher than two by including the approximation measure of $\varepsilon_j$ for the $r^{th}$ time derivative of $f_j$ in the buffer condition. And since affine buffers $\mathcal{B}$ and $\mathcal{B}_J$ are repulsive, we can guarantee that the trajectories will never violate constraint \eqref{eq: constraint}, even in the presence of discrete jumps, and the overall safety of the hybrid system can be guaranteed.
\begin{corollary} \label{thm:coro_overall_safety}\vspace{-5pt}
If the control policy $\pi_\theta$ is designed such that the resulting closed-loop trajectories satisfy the conditions of Theorem~\ref{thm:theorem_policed} and Theorem~\ref{thm:theorem_reldeg2} for all states $s\in S$ across all discrete modes $q \in Q$, then the resulting hybrid system never violates safety constraint~\eqref{eq: constraint}.
\end{corollary}
\vspace{-5pt}
Note that for simplicity, we presented our methodology for a single constraint and a single constraint-violating jump set. Theorem \ref{thm:theorem_policed} and Theorem \ref{thm:theorem_reldeg2} can be easily extended to multiple constraints by designing multiple buffers $\mathcal{B}$ and $\mathcal{B}_J$. Having established the theoretical conditions under which an RL policy can guarantee constraint satisfaction, we now describe how to implement this framework in practice.
\vspace{-2pt}
\subsection{Training Pipeline}
The authors of \cite{POLICEd_RL, CDC_POLICEd_RL} used POLICE Deep Neural Networks (P-DNN) to make the policy affine in the buffer region $\mathcal{B}$. P-DNNs utilize deep network spline theory to force the network outputs to be affine within a region.
For hybrid systems, where we require at least two affine regions $\mathcal{B}$ and $\mathcal{B}_J$, P-DNNs often fail as having multiple affine regions greatly reduces the network's expressivity. To overcome this limitation, we employ a switched actor network architecture comprised of three actor networks: an affine actor for buffer $\mathcal{B}$, an affine actor for buffer $\mathcal{B}_J$, and a nonlinear actor (MLP) for the remaining state space. All actors share a common critic network, which ensures consistent value estimation across regions. Our affine actors are single-layer neural networks without any activation function, ensuring affine outputs. We can also easily satisfy more constraints by adding additional buffer networks. For a given state $s$, our final policy $\pi_{\theta}(s)$ selects the output from the appropriate actor. 

This architecture ensures that the policy adheres to the affine constraints in both buffer regions while maintaining flexibility in the rest of the state space. The affine actors are trained to satisfy the repulsion conditions \eqref{eq: repulsive} or \eqref{eq: repulsive_reldeg2}, while the nonlinear actor focuses on maximizing the task objective in the unconstrained regions. Training proceeds in two stages. In the first stage, we train a base policy using a standard RL algorithm with a single nonlinear actor-critic pair, optimizing solely for the task objective without any safety considerations. In the second stage, we construct our switched-actor architecture and use a frozen copy of our prior policy as our nonlinear network to preserve learned task performance. The affine actors are trained by repeatedly resetting to states within their respective buffer regions. Training continues until the repulsion conditions are satisfied at all vertices of the buffer regions.

Our buffer actors are single-layer networks without activation functions, and so our policy within the buffer regions is affine by construction. During training, we initialize episodes inside the corresponding buffer and penalize constraint violations in the reward to guide the optimization toward a suitable affine policy. This penalty is just a training heuristic and does not establish safety. The formal guarantee instead requires verifying the repulsion or dissipation condition in Theorem~\ref{thm:theorem_policed} or Theorem~\ref{thm:theorem_reldeg2}, respectively.
Once these conditions are met, our learned policy has formal constraint satisfaction guarantees for hybrid systems with affine reset maps. In the next section, we demonstrate how our two-stage training process effectively balances task performance and safety.

\begin{figure}[!htb]
\centering
\begin{tikzpicture}[
    scale=0.85,
    box/.style={rectangle, draw, fill=gray!20, minimum height=0.7cm, minimum width=2.0cm, font=\small},
    input/.style={rectangle, draw, fill=blue!15, minimum height=0.7cm, minimum width=1.3cm, font=\small},
    output/.style={rectangle, draw, fill=green!20, minimum height=0.7cm, minimum width=2.2cm, font=\small},
    circ/.style={circle, draw, fill=white, minimum size=0.9cm, inner sep=0pt},
    out_circ/.style={circle, draw, fill=orange!25, minimum size=0.9cm, inner sep=0pt, font=\small},
    arrow/.style={->,>=latex,thick}
]

    
    \node[input] (state) at (0, 8) {$s$};
    
    \node[circ] (select) at (-1.5, 6) {select};
    
    \node[output] (critic) at (1.5, 6) {Value};
    \node[out_circ] (value_out) at (4.0, 6) {$V(s)$};
    
    \node[box] (bufferB) at (-3.5, 4) {Affine Policy for $\mathcal{B}$};
    \node[box] (bufferBJ) at (0, 4) {Affine Policy for $\mathcal{B}_J$};
    \node[output] (actor) at (3.5, 4) {MLP Policy};
    

    
    \draw[arrow] (state.south) -- ++(0, -.5) -| (select.north);
    \draw[arrow] (state.south) -- ++(0, -.5) -| (critic.north);
    
    \draw[arrow] (select.south) -- ++(0, -0.5) -| (bufferB.north);
    \draw[arrow] (select.south) -- ++(0, -0.5) -| (bufferBJ.north);
    \draw[arrow] (select.south) -- ++(0, -0.5) -| (actor.north);
    
    
    \draw[arrow] (critic.east) -- (value_out.west);

\end{tikzpicture}
\caption{Actor-Critic Network Architecture (Modified Flow)}
\label{fig:actor_critic_mod}
\end{figure}
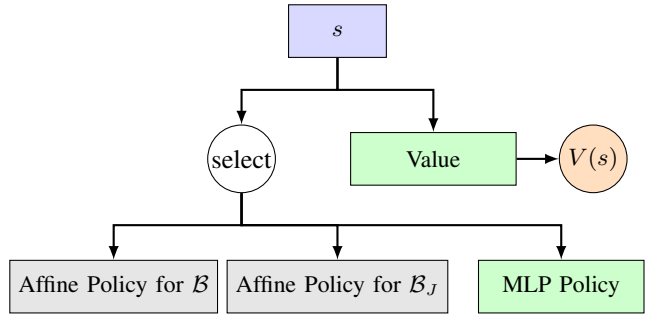

\vspace{-4pt}
\section{Numerical Simulations}\label{numerical_simulations}
\vspace{-3pt}
In this section, we evaluate our approach on two representative hybrid dynamical systems and compare with the soft constraint method Constrained Policy Optimization (CPO) \cite{CPO} and the learned control barrier function PPO-Barrier \cite{PPO-Barrier}. 

\subsection{Constrained Pendulum} 

\begin{figure}[!t]
\vspace*{4pt}
    \centering
    \begin{tikzpicture} [scale = 0.5]
        \filldraw[black] (2,1) circle (6pt);
        \filldraw[black] (0,5) circle (2pt);
        \filldraw[black] (-2.0,2) circle (3pt);
        \draw[ultra thick] (2, 1) -- (0, 5);
        \draw[thin, dashed] (0, 0) -- (0, 5);
        \draw[thin, dashed] (0, 5) -- (-2.0,2);
        \node at (1.4, 1) {$m$};
        \node at (-2.7, 2) {$pin$};
        \node at (-1.6, 3.5) {$l_p$};
        \node at (1.5, 3) {$l$};
        \node at (-0.5, 3.0) {$\phi_p$};
        \node at (0.5, 2.0) {$\phi$};
        \draw[thin, ->] (0,3.5) .. controls (-0.5,3.4) .. (-1.0,3.5);
        \draw[thin, ->] (0,2.5) .. controls (0.5,2.3) .. (1.2,2.6);
        \draw[thin, ->] (2.9, 1) -- (2.3, 1);
        \node at (2.5,1.3) {$u$};
    \end{tikzpicture}
    \begin{tikzpicture} [scale = 0.5]
        \filldraw[black] (-3.2,1.2) circle (6pt);
        \filldraw[black] (0,5) circle (2pt);
        \filldraw[black] (-2.0,2) circle (3pt);
        \draw[ultra thick] (-2, 2) -- (0, 5);
        \draw[ultra thick] (-3.2,1.2) -- (-2, 2);
        \draw[thin, dashed] (0, 0) -- (0, 5);
        \draw[thin, dashed] (0, 5) -- (-2.0,2);
        \draw[thin, dashed] (-2, 2) -- (-2.0,0);
        \node at (-3.2, 0.7) {$m$};
        \node at (-1.2, 2) {$pin$};
        \node at (-1.5, 3.5) {$l_p$};
        \node at (-3.3, 2.0) {$l_s$};
        \node at (-0.5, 3.0) {$\phi_p$};
        \node at (-2.5, 1.2) {$\phi$};
        \draw[thin, ->] (0,3.5) .. controls (-0.5,3.4) .. (-1.0,3.5);
        \draw[thin, ->] (-2.0,1.5) .. controls (-2.3,1.3) .. (-2.5,1.6);
        \draw[thin, ->] (-4.1, 1.2) -- (-3.5, 1.2);
        \node at (-3.9, 0.9) {$u$};
    \end{tikzpicture}
    \caption{Constrained Pendulum System: (Left) Mode $q_1$, (Right) Mode $q_2$}
    \label{fig: constrained_pendulum_diagram}
    \vspace{-12pt}
\end{figure}
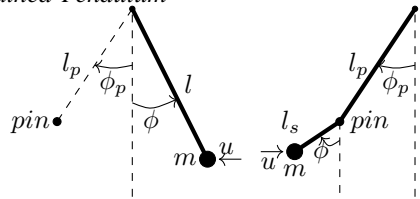

For our first scenario, we consider a constrained pendulum of length $l$ that hits a pin during its motion. The pin instantaneously alters the effective length of the pendulum to $l_s$, introducing discrete events, as illustrated in Fig.~\ref{fig: constrained_pendulum_diagram}. Due to the conservation of angular momentum, this event produces a sudden jump in the pendulum's angular velocity. Consequently, this hybrid system exhibits two distinct modes $q_1$ and $q_2$, as illustrated in Fig.~\ref{fig: constrained_pendulum_diagram}. We define the continuous state vector $s = [\phi, \dot{\phi}]^\top$, where $\phi$ denotes the pendulum's angle with the vertical and $\dot{\phi}$ its angular velocity. When the pendulum reaches the pin, which is located at an angle with the vertical $\phi_p = -\pi/12$, the pendulum's angular velocity changes instantaneously by a factor of $\eta = l/l_s =3.33$, the ratio of the length of the pendulum in $q_1$ to $q_2$.


We also assume that the controller has no access to the underlying dynamics. The pendulum starts in mode $q_1$. The goal of the controller $u$ which is acting on the pendulum's bob, is to reach angle $-\frac{\pi}{2}$, with the angular velocity constraint $\dot{\phi} < \dot{\phi}_{\max}= -5 $rad/s when the state $s$ is in mode $q_2$. This condition indicates that the pendulum string is not wrapped around the pin during oscillation.



We formally define our constraint as $y(t) := \dot{\phi} \leq \dot{\phi}_{\max}$, in mode $q_2$. If the angular velocity is greater than  $\frac{\dot{\phi}_{\max}}{\eta}$, then after reset, the constraint $y$ will be violated. So we define constraint $y_J(t)$ such that  $y_J(t):= \dot{\phi} \leq \frac{\dot{\phi}_{\max}}{\eta}= -1.5$ rad/s for states before reset in mode $q_1$. This constraint $y_J(t)$ has relative degree 1 with respect to the control input $u$. 
We use a switched actor network with 2 affine actors to satisfy $y(t)$ and $y_J(t)$, and a 3-layer MLP actor to achieve the task of reaching $-\frac{\pi}{2}$. {We estimate the approximation measure by least squares on samples drawn from each buffer, obtaining $\varepsilon = 0.1 \text{rad}/s^2$ for both $\mathcal{B}$ and $\mathcal{B}_J$. We set the buffer width $w=0.2$ rad/s, which is wide enough that no trajectory crosses either of the buffers within one control period. We train our policy using TD3 \cite{TD3} as the base RL algorithm for about $1000$ iterations, followed by training the affine actor for buffers $\mathcal{B}$ and $\mathcal{B}_J$ by resetting the state in buffers until the condition given by Theorem~\ref{thm:theorem_policed} is met or $500$ iterations are elapsed.

Fig.~\ref{fig:cp_results} compares our policy with a base TD3 policy. The red line denotes our velocity constraint $y(t)$ ($\dot \phi >-5$ rad/s), and the orange line denotes the jump constraint $y_J(t)$ ($\dot \phi>-1.5$ rad/s) with $\eta=3.33$. The base TD3 policy violates constraints both when the state is close to the boundary or during velocity jumps, and due to a velocity jump when the pre-reset velocity is less than $-1.5$ rad/s. In contrast, our policy remains safe throughout and exhibits repulsive behavior within the buffer regions while still reaching the target angle of $-\frac{\pi}{2}$.

\begin{figure}[!t]
    \centering
        \centering
    \begin{tabular}{@{}c@{\hspace{2pt}}c@{}}
  (a) &   \includegraphics[width=0.4\textwidth]
     {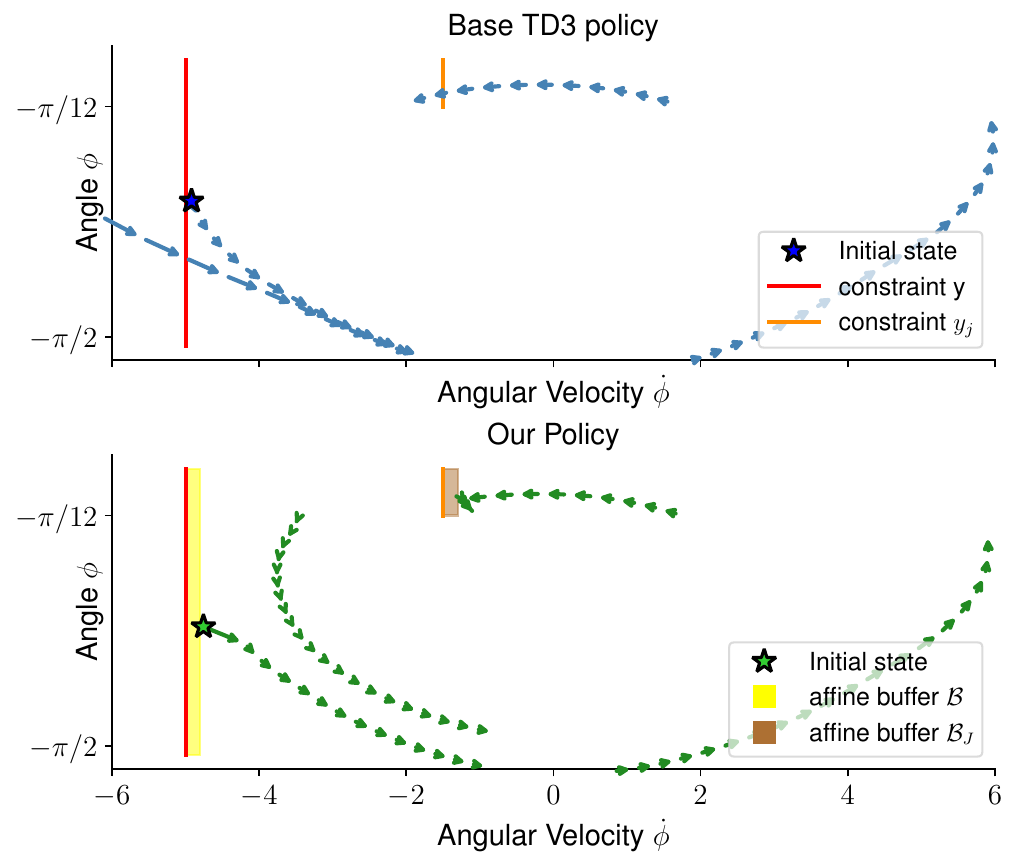}\\

        \centering

    (b) & \includegraphics[width=0.4\textwidth]
     {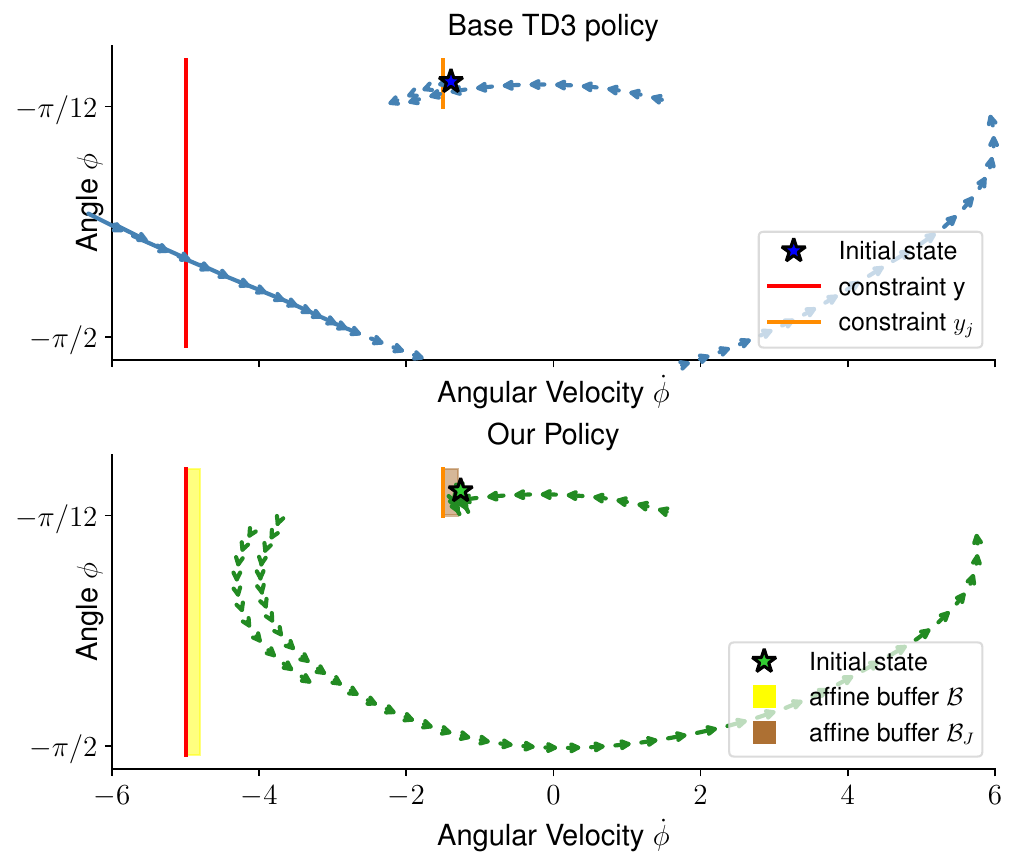}
\end{tabular}
     \vspace{-5pt}
    \caption{Phase portrait of the constrained pendulum showing angle $\phi(t)$ vs. angular velocity $\dot{\phi}(t)$ under TD3 and our policies. (a) Trajectories starting near $y(t)$ and (b) Trajectories starting near $y_J(t)$, unlike TD3, remain safe under our policy. Here, \textcolor{myRed}{Red} and \textcolor{myOrange}{Orange} lines denote the constraints $y(t)$ and $y_J(t)$, respectively. Yellow and Brown regions denote buffers $\mathcal{B}$ and $\mathcal{B}_J$ preceding constraints $y(t)$ and $y_J(t)$, respectively. }
    \label{fig:cp_results}
    \vspace{-15pt}
\end{figure}


\vspace{-4pt}
\subsection{Paddle Juggler}
\vspace{-5pt}
For our second scenario, we consider a one-dimensional paddle juggler system in which our goal is to juggle a ball to its maximum apex height. This system exemplifies hybrid dynamics as energy transfer occurs only at the instantaneous moment that the paddle and ball are in contact. The continuous dynamics are given by $\ddot{x}_b = -g, \ddot{x}_p = u,$ where $g$ is the gravity and $u$ is the force acting on the paddle at all times. Thus, the hybrid system has a single mode $q$ with continuous state, $x = [x_b, \dot{x}_b, x_p, \dot{x}_p]^\top$ containing the ball's and paddle's vertical positions and velocities. We assume the paddle's mass is very large, so the impact has no effect on its dynamics. When the position of the ball $x_b$ reaches the paddle position $x_p$, $\dot {x}_b$ changes instantaneously by a factor of the restitution coefficient  $e \in [0,1]$. 

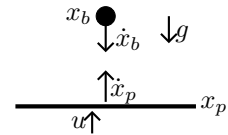
\begin{figure}[!htb]
\tikzset{
ctrl/.style       ={-{Straight Barb[length=3pt]}, very thick},
}
    \centering
    \begin{tikzpicture}[scale = 0.6]
        \filldraw[black] (0,3) circle (6pt);
        \node at (0.9, 3) {$x_b$};
        \draw[thick, -{Straight Barb[length=3pt]}] (0, 3) -- (0, 2);
        \node at (0.4, 2.5) {$\dot x_b$};
        \draw[ultra thick] (-2, 0) -- (2, 0);
        \node at (2.4, 0) {$x_p$};
        \draw[thick, -{Straight Barb[length=3pt]}] (0, 0.1) -- (0, 0.8);
        \node at (0.4, 0.35) {$\dot x_p$};
        \draw[thick, -{Straight Barb[length=3pt]} ] (1, 2) -- (1, 1.4);
        \node at (1.3, 1.6) {$g$};
        \draw[thick,-{Straight Barb[length=3pt]}] (-0.3, -0.6) -- (-0.3, -0.1);
        \node at (-0.6, -0.35) {$u$};
    \end{tikzpicture}
    \caption{Paddle Juggler System}
    \label{fig: ball_paddle_system}
\end{figure}

To properly simulate our system and design our constraint, we apply a coordinate system transformation where $s[1]$ and $s[2]$ denote the relative distance and velocity between the paddle and ball, while $s[3]$ and $s[4]$ denote the paddle's position and velocity. The dynamics still remain unknown from the controller's perspective. The control objective is to maximize the ball's apex height while satisfying the relative velocity constraint $s[2] = \dot{x}_b - \dot{x}_p  < \dot{s}_{\max}$, which ensures the system does not become unstable due to high momentum.

\begin{figure}
    \centerline{
        \includegraphics[width= 0.5\textwidth]{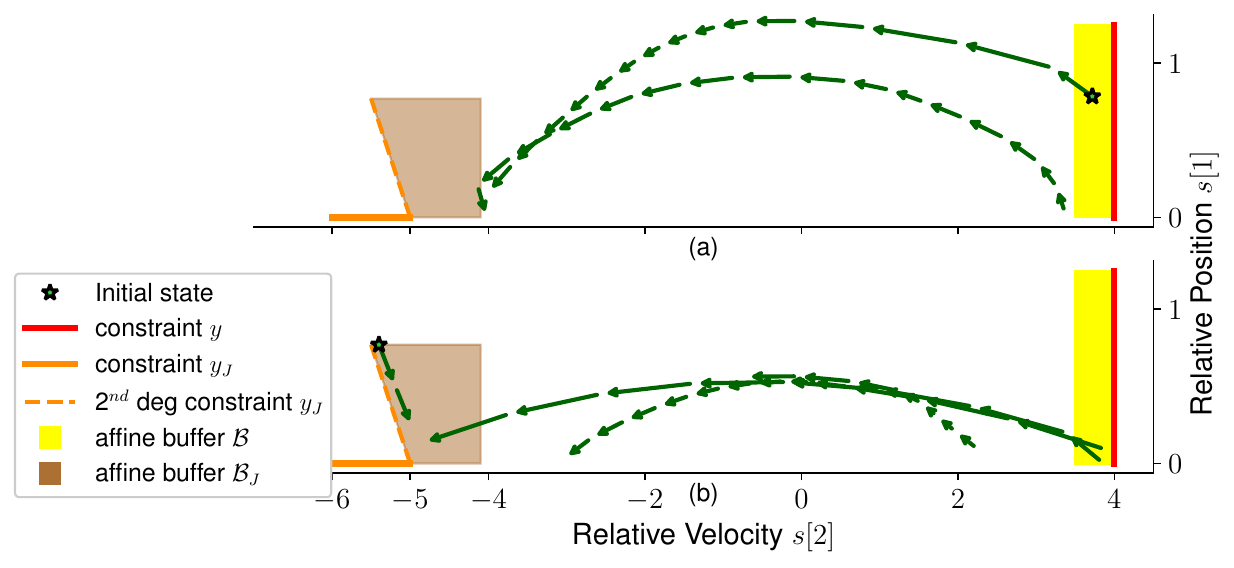}}
        \caption{Phase portrait of relative position $s[1](t)$ and relative velocity $s[2](t)$ for the paddle juggler under our policy. (a) Trajectories starting near $y(t)$ and (b) Trajectories starting near $y_J(t)$, remain safe under our policy. Here, \textcolor{myRed}{Red} and \textcolor{myOrange}{Orange} lines denote the constraints $y(t)$ and $y_J(t)$, respectively. Yellow and Brown regions denote the buffers $\mathcal{B}$ and $\mathcal{B}_J$, respectively. }
    \vspace{-20pt}
    \label{fig:pj_results}
\end{figure}

We formally define the constraint as $y(t) := s[2] \leq \dot{s}_{\max}$. Thus, we want to prevent state resets when the relative velocity is less than $-\dot s_{\max}/e$. The impact constraint $y_J(t)$ has relative degree two with respect to the control input $u$.
Alternatively, we could use a more restrictive proxy first-order constraint $\dot{s} \geq -\dot{s}_{\max}/e$, which simplifies learning. In this work, we employ the high-relative-degree formulation to illustrate the capability of our method. The network architecture and training pipeline are the same as in the constrained pendulum. Because our continuous dynamics $\ddot{x}_b = -g$ and $\ddot{x}_p = u$
are linear, the affine approximation~\eqref{eq: epsilon} is exact and
$\varepsilon = 0$ for both buffers. We set the buffer width to $w = 0.5$~m/s for
buffer $\mathcal{B}$. We design $\mathcal{B}_J$ as a triangular buffer, spanning a width of $0.5$ along $y_J$ with $y_{J\min} = 0.77$. The only difference is in training the affine policy for buffer $\mathcal{B}_J$, where we try to satisfy the condition of Theorem~\ref{thm:theorem_reldeg2} by randomly resetting the initial state within the buffer for $1000$ iterations.

Fig.~\ref{fig:pj_results}(a) shows that even when trajectories start close to the relative velocity constraint $\dot{s}_{\max} = 4$, the affine policy in buffer $\mathcal{B}$ drives them away from the constraint $y$. If the ball before the reset has a relative velocity less than $-5$ with $e=0.8$, the post-reset relative velocity will violate our constraint $y$. We therefore design a dissipative affine buffer $\mathcal{B}_J$ that pushes trajectories away from $y_J$, as shown in Fig.~\ref{fig:pj_results}(b).
\subsection{Comparison with Baselines}
Table~\ref{tab:tab_results} shows the comparison of Average Constraint Satisfaction Percentage (ACS\%), which is the percentage of episodes where the constraint is never violated, irrespective of task completion, to evaluate the safety capabilities of the policy. Completion without Constraint Violation Percentage (CCV\%) represents the percentage of episodes where the policy achieves the target without violating the constraint across 100 rollouts (50 rollouts initialized near the constraint boundary and close to reset, and 50 rollouts initialized far from the constraint). Our algorithm is the only one to guarantee constraint satisfaction for the above numerical simulations, whereas the CPO \cite{CPO} and PPO-Barrier \cite{PPO-Barrier} baselines, and TD3 \cite{TD3}, cannot achieve this task without violating the safety constraints in both of the above environments. 


\begin{table}[!htb] 
\vspace*{4pt}
\centering
\footnotesize
\setlength{\tabcolsep}{2pt}
\renewcommand{\arraystretch}{1.05} 
\begin{tabular}{c >{\centering\arraybackslash}b{1.4cm} >{\centering\arraybackslash}b{1.4cm} >{\centering\arraybackslash}b{1.4cm} >{\centering\arraybackslash}b{1.4cm}}
\noalign{\hrule height 1.2pt} 
&\multicolumn{2}{c}{\textbf{Constrained Pendulum}} 
& \multicolumn{2}{c}{\textbf{Paddle-Juggler}} \\
\cmidrule(lr){2-3}
\cmidrule(lr){4-5}
\textbf{Algorithm}
  & \textbf{ACS \% $\uparrow$} 
  & \textbf{ CCV \% $\uparrow$}
  & \textbf{ACS \% $\uparrow$} 
 & \textbf{ CCV \% $\uparrow$} \\
\hline
TD3& 99 & 80& 99 & 50 \\
CPO& 94 & 18 & 90 & 45 \\
PPO Barrier& 99& 91 & 99 & 49\\
 \hline
\textbf{Ours} & \textbf{100} & \textbf{100} & \textbf{100} & \textbf{100} \\
\noalign{\hrule height 1.2pt} 
\end{tabular}
\caption{Comparison of the Average Constraint Satisfaction(ACS) \% and the Completion without Constraint Violation(CCV) \% of our policy with baselines: TD3~\cite{TD3}, CPO~\cite{CPO}, and PPO-Barrier~\cite{PPO-Barrier} for 100 rollouts (50 rollouts initialized near the constraint boundary and close to reset, and 50 rollouts initialized far from the constraint).}
\vspace{-10pt}
\label{tab:tab_results}
\end{table}

\section{CONCLUSION AND FUTURE WORK} \label{conclusion}
\vspace{-1pt}

We proposed a framework that guarantees the satisfaction of safety constraints in closed loop for black-box hybrid dynamical systems. Our key insight is to design affine repulsive buffers in our policy networks around safety constraints, ensuring that trajectories never violate the constraints, even when involving instantaneous state jumps. We demonstrate the effectiveness of our approach on complex hybrid systems like the constrained pendulum and paddle juggler, where our method achieves 100\% constraint satisfaction, compared to both soft penalty and learned-CBF baselines. One of the limitations of our work is that choosing a buffer width $w$ and approximation measure $\epsilon$ jointly is non-trivial since a wider buffer leaves a larger safety margin but forces a coarser affine fit and a larger $\epsilon$, which in turn tightens our repulsion condition. In future work, we aim to develop a method of finding optimal $w$ and $\epsilon$. We also aim to extend this framework to higher-dimensional systems (e.g., humanoid locomotion) where strong nonlinearities make finding locally affine policies challenging. 


\vspace{-6pt}
\bibliographystyle{IEEEtran}
\bibliography{references.bib}

\section{Appendix}\label{supplemental}
\subsection{Buffer $\mathcal{B}$ is a convex polytope}
\begin{lemma}
Buffer $\mathcal{B}$ is a polytope.\cite{POLICEd_RL}
\end{lemma}
\begin{proof}
We can write buffer $\mathcal{B}$ as $\mathcal{B} = C'([d-r,d]) \cap S_i$ where $C'([d-r,d])$ denotes the inverse image of the interval $[d-r,d]$, meaning $C'([d-r,d]) := \{s : Cs \in [d-r,d]\}$. Note that the inverse image of an interval always exists, even if the matrix $C$ is not invertible. Therefore, $\mathcal{B}$ is the intersection of the affine variety $C'([d-r,d])$ and the polytope $S_i$, and hence $\mathcal{B}$ is a polytope according to Theorem 3.1.4 of \cite{Grunbaum2003ConvexPolytopes}. 
\end{proof}

\subsection{Proof of Theorem~\ref{thm:theorem_reldeg2}}
\begin{proof}
Since our policy is affine in $\mathcal{B}_J$, we can denote it as $\pi_\theta(s) = D_\theta\ s + e_\theta$ for any state $s$ in $\mathcal{B}_J$. Additionally, $\epsilon_j$ is an approximation measure and there exists $A \in \mathbb{R}^{n \times n}$, $B \in \mathbb{R}^{n \times m}$ and $c \in \mathbb{R}^{n}$ which defines the affine approximation of the second derivative of dynamics $f_j$ in buffer $\mathcal{B}_J$. The following condition holds for $s = v \in \mathcal{V}(\mathcal{B}_J)$:
\begin{equation}
     \begin{aligned}
     \Tilde{C}(Av + B\pi_{\theta}(v) + c)   & \leq |\Tilde{C}(Av + B\pi_{\theta}(v) + c) \\ 
                                            &\quad -\Tilde{C}\ddot{f}_j(v,\pi_{\theta}(v))| \\
                                            &\quad + \Tilde{C}\ddot{f}_j(v,\pi_{\theta}(v)) \\
     & \leq -\varepsilon_j - \beta \dot{v} .
     \end{aligned}
\end{equation}
where the first inequality is directly from the triangular inequality and characteristics of the norm, the second follows from our definition of $\epsilon_j$ and our assumed condition~\eqref{eq: repulsive_reldeg2}. Using the convexity of polytope $\mathcal{B}_J$ of vertices $\mathcal{V}(\mathcal{B}_J) = \{v^1, ..., v^N\}$, for any $s \in \mathcal{B}_J$, there exists $\alpha_1, \alpha_2, ..., \alpha_N \in \mathbb{R}^+$ such that $\sum\limits_{k=1}^{N}\alpha_k = 1$ and $s = \sum\limits_{k=1}^{N}\alpha_kv^k$ where $N$ is the number of vertices of $\mathcal{B}_J$. Policy ~\eqref{eq: expected reward} applied at any $s \in \mathcal{B}_J$ thus yields

\begin{align*}
    \Tilde{C}(As + B\pi_{\theta}(s) + c)  &= \Tilde{C}(As + B(D_{\theta}\ s + e_{\theta}) + c) \\
    &= \Tilde{C}((A+BD_{\theta})s + Be_{\theta} + c) \\
    &= \Tilde{C}\left ((A+BD_{\theta})\sum\limits_{k=1}^{N}\alpha_kv^k \right) \\
    &\quad + \Tilde{C}(Be_{\theta} + c )\sum\limits_{k=1}^{N}\alpha_k \\
    &= \sum\limits_{k=1}^{N}\Tilde{C} \left ((A+BD_{\theta})v^k + Be_{\theta}+ c \right )\alpha_k \\
    &= \sum\limits_{k=1}^{N}\alpha_k\Tilde{C}\left (Av_k + B(D_{\theta}v^k + e_{\theta})+ c\right ) \\
    &= \sum\limits_{k=1}^{N}\alpha_k\Tilde{C}\left (Av_k + B\pi_{\theta}(v^k)+ c\right ) \\
    &\leq \sum\limits_{k=1}^{N}\alpha_k(-\varepsilon_j) -   \beta  \sum\limits_{k=1}^{N}\alpha_k \dot{v}^k \\
    &\leq -\varepsilon_j - \beta \dot{s}. \numberthis \label{eq: convex_BJ_appendix}
\end{align*}

And similarly to earlier, we note that
\begin{align*}
    \ddot{y}_J &= \Tilde{C}\ddot{f}_j(s,\pi_{\theta}(s)) \\
    & \leq |\Tilde{C}\ddot{f}_j(s,\pi_{\theta}(s)) - \Tilde{C}(As + B\pi_{\theta}(s) + c)| \\
    &\quad + \Tilde{C}(As + B\pi_{\theta}(s) + c) \\
    & \leq \varepsilon_j - \varepsilon_j - \beta \dot{s} \\
    & \leq - \beta \dot{y}_J. \numberthis \label{eq: Cs_dot_neg_2_appendix}
\end{align*}
where we first use the triangular inequality, then affine approximation measure $\varepsilon_j$, and \eqref{eq: convex_BJ_appendix}. Having proved \eqref{eq: Cs_dot_neg_2_appendix}, we will show that it prevents all trajectories $\mathcal{\tau}^S(s_0,\pi_{\theta})$ from entering unsafe jump set $J_{ji}$ when $s_0 \in S_s$ where $S_s$ is a safe set. Using equation \eqref{eq: Cs_dot_neg_2_appendix}, we can state that the function $y_J^{k}$ for $k \in \{1, 2\}$ are always decreasing can be given by
\begin{equation}
    y_J^{(k)}(t)\leq y_J^{(k)}(t_0)e^{-\beta(t - t_0)} < -\beta y_J^{(k - 1)}(t_0)e^{-\beta(t - t_0)},
\end{equation}
where $t_0$ is the time the trajectory enters the buffer $\mathcal{B}_J$, the $y_J^{(k)}$ decreases with time and hence, the trajectory reaches $y_J= \Tilde{d}$ at time $T$ with $\dot y(T) \leq 0$. Thus, the trajectories cannot leave the buffer $\mathcal{B}_J$ through the constraint and thus it is provable impossible for the trajectory to reach the constraint-violating jump set $J_{ji}$ (which after a state reset would have violated the constraint $y(t)$).
\end{proof}

\subsection{Buffer Design for Relative Degree 2 Buffer}
The purpose of the buffer is to provide a region of the state space in which the controller dissipates the generalized inertia of trajectories to prevent violation of the constraint 
\begin{equation} \label{eq:proof_constaint_j}
y_J(t) = \Tilde{C}s(t) \leq \Tilde{d}.
\end{equation}



To dissipate inertia and guarantee that trajectories reaching the boundary $y_J=\Tilde{d}$ cannot continue towards it, we allow positive values of $\dot y_J$ in the interior but require that $\dot y_J \le 0$ when $y_J=\Tilde{d}$.  Requiring $\dot y_J \leq 0$ for all $s \in \mathcal{B}_J$ is not feasible, since it does not consider the states we need to slow down before the constraint boundary $y_J=\Tilde{d}$. Therefore, we allow $\dot{y}_{Jmax} > 0$ to be the maximal velocity which can be dissipated by the buffer $\mathcal{B}_J$. To ensure this property while maintaining convexity, we impose an affine upper bound on $\dot y_J$:
\begin{equation} \label{eq:inequality_appendix}
    \dot y_J \le \beta (\Tilde{d} - y_J), \quad\text{with} \quad \beta := \frac{\dot y_{Jmax}}{\Tilde{d}-y_{Jmin}} 
\end{equation}
and $\dot y_{Jmax}>0$ is chosen so that states at $y_J=y_{Jmin}$ may have maximal velocity $\dot y_{Jmax}$ and $x_{2max} = 0$. We choose $y_{Jmin} < 0$ so that $y_{Jmin} < x_{Jmax}$ for all $y_J$. To enforce the inequality ~\eqref{eq:inequality_appendix}, we need to actuate the $\ddot{y}_J$ since the relative degree between $y_J$ and $u$ is $2$. Differentiating ~\eqref{eq:inequality_appendix} with time we get
\begin{equation}
\ddot y_J \le -\beta \dot{y}_J
\end{equation}

Thus, we choose buffer $\mathcal{B}_J$'s lower bound $\underline{b}$ and upper bound $\bar{b}$ to be as follows:

\begin{align*}
\underline b = \left [ y_{Jmin},  \quad \dot{y}_{Jmin} \right ] \\
  \bar b = \left [ y_{Jmax},  \quad \beta(y_{Jmax}-s_1) \right ]
\end{align*}

where $s_1$ denotes the first component of state $s \in \mathcal{B}_J$. The remaining coordinates of states in state space do not influence constraint satisfaction. For these, we select a compact polytope $\mathcal{P} \subset\mathbb{R}^{n-2}$.

The buffer region is then defined as
\begin{equation}
    \mathcal{B}_J := \left\{\, s\in S_j :\; s_{1:2}\in [\,\underline b,\; \bar b(s)\,],\;\; s_{3:n}\in \mathcal{P} \right\}.
\end{equation}
Note that by design $\mathcal{B}_J$ is a convex polytope. 

\subsection{Numerical Simulations Details}
\subsubsection{Constrained Pendulum}

We define the continuous state vector as $s = [\phi, \dot{\phi}]^\top$, where $\phi$ denotes the pendulum's angle with the vertical line and $\dot{\phi}$ its angular velocity. The dynamics governing continuous state transitions within mode $q_1$ are given by:
\begin{equation} \label{eq:pend_dynamics}
\begin{aligned}
  \ddot{\phi} =  - \frac{g}{l}sin(\phi) -\frac{z}{m}\dot{\phi} + u,\\
\end{aligned}
\end{equation}
where $g$ is the gravity coefficient, $m$ is the mass of the pendulum, $z$ is the damping factor and $u$ is the force acting at the pendulum's bob. When the pendulum reaches the pin, which is located at an angle with the vertical $\phi_p = -\pi/12$, the pendulum's angular velocity changes due to instantaneous impact. We can model the discrete dynamics that lead to a switch from mode $q_1$ to mode $q_2$ as follows:
\begin{equation}\label{eq:pend_resetmap}
\begin{aligned}
G_{12} = \phi - \phi_p \leq 0, \ \ \ R_{12}\Bigl(
\begin{bmatrix}
  \phi \\  \dot{\phi}
\end{bmatrix}
\Bigl) = \begin{bmatrix}
  \phi_p \\ \eta\dot{\phi} 
\end{bmatrix},
\end{aligned}
\end{equation}
where $\eta = l/l_s$ is the scaling of velocity due to the change in length of the pendulum and $l_s = l/3$ is the length of the pendulum with respect to the pin, i.e., $l_s = l - l_p$, and $l_p$ is the position of the pin with respect to the pendulum's fixed point. Since $l_s$ is smaller than $l$, $\eta = 3 $. The dynamics of mode $q_2$ are the same as \eqref{eq:pend_dynamics} except that the length of the pendulum is $l_s$ instead of $l$.
The discrete dynamics which leads to switch from mode $q_2$ to mode $q_1$  is $G_{21}= \phi - \phi_p  \geq 0 $ and the reset map is $R_{21}(s) = [ \phi_p, \frac{\dot{\phi}}{\eta} ]^\top$.

We assume that dynamics \eqref{eq:pend_dynamics} and \eqref{eq:pend_resetmap} are a black-box model from the controller's perspective, i.e., we do not provide the dynamics equations to the controller. The pendulum starts at an initial position $\phi(0) = \phi_0$, with zero angular velocity $\dot{\phi}(0) = 0$ in mode $q_1$. Our goal is to reach angle $-\frac{\pi}{2}$, while we require the angular velocity to be at most $\dot{\phi}_{max}$ when the state $s$ is in mode $q_2$ in order not to wrap the pendulum string around the pin while oscillating. Training a safe RL policy to satisfy this constraint is challenging since the policy needs to account for the jump in angular velocity due to the reset map $R_{12}$. Our framework can provably enforce this constraint.

We first define the constraint as $y(t) = Cs(t) \leq d \quad \forall t$ where $C=[0,\ 1]$ and $d=[0,\ -\dot{\phi}_{max}]^\top$ and $\dot{\phi}_{max} = 5$ rad/s. We also need to define another constraint for the constraint violating jump set $J_{12}$ = $\{s \in S : ((q_1,s),(q_2,s')) \in T; \ CR_{12}(s) \leq d$. Since our reset map $R_{12}(s)$ is affine we can define the constraint $y_J(t) = \Tilde{C}s(t) \leq \Tilde{d}$ for all states in $J_{12}$ where $\Tilde{C} = [0,\ \eta]$ and $\Tilde{d}= [0,\ -\dot{\phi}_{max}]^\top$.

To enforce these constraints using our framework, we define buffer regions $\mathcal{B}$ and $\mathcal{B}_J$ preceding the constraints $y(t)$ and $y_J(t)$, respectively, where the policy is affine and repulsive. We choose the buffers as $\mathcal{B} := [-\pi, 0] \times [-\dot \phi_{max},\dot \phi_{\mathcal{B}}]$. Buffer for $J_{12}$ is $\mathcal{B}_J := [\phi_p, 0] \times [-\dot \phi_{max}/\eta, \dot \phi_{\mathcal{B}_J}]$. We choose $\dot \phi_{\mathcal{B}}$ and $\dot \phi_{\mathcal{B}_J}$ such that the buffers are wide enough so that they cannot be "jumped" over by the system in a single time-step.
From \eqref{eq:pend_dynamics} and \eqref{eq:pend_resetmap}, we can infer that the constraint $y_J(t)$ has relative degree 1 with respect to the control input $u$. Therefore, we can use condition of Theorem \ref{thm:theorem_policed} with buffer $\mathcal{B}_J$ to provably satisfy the constraint $y_J(t)$, i.e., to guarantee $\eta\dot{\phi} \leq -\dot \phi_{max}$. 


\subsubsection{Paddle Juggler}

The one-dimensional paddle juggler system is the hybrid system with a single discrete mode $q$ with continuous state, $x = [x_b, \dot{x}_b, x_p, \dot{x}_p]^\top$ containing the ball's and paddle's vertical positions and velocities. The guard condition is $G = x_b - x_p \leq 0$, and the reset map is $R(s) = [x_b,\ (1+e)\dot{x}_p -e\dot{x}_b, \ x_p, \ \dot{x}_p]^\top$ where $e \in [0,1]$ is the coefficient of restitution. 

For designing constraints that are agnostic to the position of the ball and the paddle, we apply a coordinate system transformation to consider the relative distance between the paddle and the ball. We define the new states as $s[1] := x_b - x_p$, $s[2] := \dot{x}_b - \dot{x}_p$, $s[3] := x_p$ and $s[4] := \dot{x}_p$. The continuous and the discrete dynamics of the system in these relative coordinates are therefore given by:
\begin{align}\label{eq:pj_dynamics}
\dot{s}[2] = -g - u, \ \ \ \dot{s}[4] = u,
\end{align}
\begin{equation}\label{eq:pj_discrete}
\begin{aligned}
G(q,q) = s[1] \leq 0 , \ \ 
R\Bigl(
\begin{bmatrix}
  s[1] \\ s[2]  \\ s[3]  \\ s[4] 
\end{bmatrix}
\Bigl) = \begin{bmatrix}
  s[1]  \\ -es[2]   \\ s[3]  \\ s[4] 
\end{bmatrix}
\end{aligned}
\end{equation}

We assume that dynamics \eqref{eq:pj_dynamics} and \eqref{eq:pj_discrete} remain a black box from the controller's perspective, and use them only to simulate the system. The ball starts at an initial position $x_b(0) = x_0$, with zero velocity $\dot{x}_b(0) = 0$, and the paddle's initial position is $x_p(0) = 0$, with zero velocity $\dot{x}_p(0) = 0$. Our control objective is to maximize the ball's apex height, while we require the relative velocity between the ball and the paddle to be at most $\dot{s}_{max}$. This constraint ensures the system does not become unstable due to high momentum. We demonstrate that our framework can achieve the task of juggling the ball to its maximum height while never violating the safety constraint.

We first define the constraint as $y(t) := s[2] \leq \dot{s}_{max}$ with $\dot{s}_{max}= 4$. We want to prevent the states from resetting when the relative velocity is less than $-\dot s_{max}/e$ and $e=0.8$. Thus, the two constraints are $y(t) = Cs(t) \leq d  \quad \forall t$ where $C=[0,\ 1,\ 0,\ 0]$ and $d=  [0, \dot{s}_{max}, 0, 0]^\top$. The second constraint for the jump set $J(q,q)$ = $ \{s \in S : ((q,s),(q,s')) \in T; \ CR(s) \leq d\}$. Since our reset map $R(s)$ is affine, for all states in $J(q,q)$, we can define the constraint $y_J(t) = \Tilde{C}s(t) \leq \Tilde{d}$ where $\Tilde{C}= [0,\ -e,\ 0,\ 0]$ and $\Tilde{d}= [0,\ \dot{s}_{max},\ 0,\ 0]^\top$.

The jump constraint $y_J(t)$ has relative degree $2$ with respect to the control input $u$. Consequently, we apply the condition of Theorem~\ref{thm:theorem_reldeg2} to make the buffer $\mathcal{B}_J$ repulsive and guarantee that any trajectory entering the buffer dissipates its velocity before reaching the jump constraint. To design $\mathcal{B}_J$, we introduce a relative-degree-2 buffer immediately preceding the constraint $y_J$ to dissipate the system's inertia before impact. Additionally, we can employ a relative-degree-1 buffer before the relative-degree-2 buffer to facilitate the smoother dissipation of momentum. 

Accordingly, we choose the buffers as $\mathcal{B} := [0, s[1]_{max}] \times [s[2]_{\mathcal{B}}, \dot s_{max}] \times [s[3]_{min}, s[3]_{max}] \times [ s[4]_{min}, s[4]_{max}] $ and $\mathcal{B}_J := [0, s_{\mathcal{B}_J}] \times [ -s[2]_{\mathcal{B}_J}, -e\dot s_{max}] \times [s[3]_{min}, s[3]_{max}] \times [ s[4]_{min}, s[4]_{max}] $ where $s[3]_{min}$ and $s[3]_{max}$ are the minimum and maximum value of paddle position respectively and $s[4]_{min}$ and $s[4]_{max}$ are the minimum and maximum value of paddle velocity respectively. We choose $s[2]_{\mathcal{B}}$ such that the buffers wide enough such that it cannot be "jumped" over by the system in a single time-step. We choose $s_{\mathcal{B}_J}$ and $s[2]_{\mathcal{B}_J}$ such that the inertia of the system can be dissipated before the constraint. We find $s[2]_{\mathcal{B}_J}$ by simulating the maximum relative velocity that can be dissipated within buffer $\mathcal{B}_J$ given the dynamics \eqref{eq:pj_dynamics}.
The buffers $\mathcal{B}$ and $\mathcal{B}_J$ used in our experiment are ad follows:
\begin{equation}
\begin{aligned}
    \mathcal{B}:= \{ & s[1]= x_b -x_p \in [0., 5.0] \\
& s[2]= \dot{x}_b - \dot{x}_p \in [3.5, 4.0] \\
& s[3]= {x}_p \in [-1.0, 1.0] \\
& s[4]= \dot {x}_p \in [-5.0, 5.0] \}.
\end{aligned}
\end{equation}

\begin{equation}
\begin{aligned}
    \mathcal{B}_J:= \{ & s[1]= x_b -x_p \in [0., 0.77] \\
& s[2]= \dot{x}_b - \dot{x}_p \in [-5.5, -5.0] \\
& s[3]= {x}_p \in [-1.0, 1.0] \\
& s[4]= \dot {x}_p \in [-5.0, 5.0] \}.
\end{aligned}
\end{equation}

\end{document}